\documentclass{article}

\usepackage{amsmath,amsthm,amssymb,amsfonts,mathrsfs}
\usepackage{bbm}
\usepackage{tensor}
\usepackage{hyperref}
\usepackage{fullpage}
\usepackage{graphicx}
\usepackage{enumerate}

\usepackage{algorithm}
\usepackage{algorithmic}

\newtheorem{definition}{Definition}
\newtheorem{example}{Example}
\newtheorem{lemma}{Lemma}
\newtheorem{theorem}{Theorem}

\newtheorem{proposition}{Proposition}
\newtheorem{corollary}{Corollary}

\newtheorem{innercustomthm}{Theorem}
\newenvironment{customthm}[1]
  {\renewcommand\theinnercustomthm{#1}\innercustomthm}
  {\endinnercustomthm}

\newcommand{\bomega}{\boldsymbol \omega}
\newcommand{\bone}{\boldsymbol 1}

\newcommand{\Rbb}{\mathbb{R}}

\newcommand{\Nbb}{\mathbb{N}}

\newcommand{\Bcal}{\mathcal{B}}

\newcommand{\Hcal}{\mathcal{H}}

\newcommand{\Lcal}{\mathcal{L}}

\newcommand{\Ocal}{\mathcal{O}}

\newcommand{\Xcal}{\mathcal{X}}

\DeclareMathOperator*{\argmin}{arg\,min}

\DeclareMathOperator*{\support}{support}

\DeclareMathOperator*{\sign}{sign}

\DeclareMathOperator*{\Exp}{\mathbb{E}}

\newcommand \al[1]{\begin{align*}
#1
\end{align*}
}

\newcommand \bra{\left\langle}
\newcommand \ket{\right\rangle}
\newcommand \braket[2]{\bra #1, #2 \ket}

\newcommand \parenth[1]{\left( #1 \right)}

\usepackage{listings} % for algorithms
\usepackage{framed}
\title{Dual Averaging on Compactly-Supported Distributions\\
And Application to No-Regret Learning on a Continuum}
\date{}
\author{Walid Krichene
}
\begin{document}
\maketitle
\thispagestyle{empty}
\pagestyle{empty}

%============================================================================================
\begin{abstract}
We consider an online learning problem on a continuum. A decision maker is given a compact feasible set $S$, and is faced with the following sequential problem: at iteration~$t$, the decision maker chooses a distribution $x^{(t)} \in \Delta(S)$, then a loss function $\ell^{(t)} : S \to \Rbb_+$ is revealed, and the decision maker incurs expected loss $\braket{\ell^{(t)}}{x^{(t)}} = \Exp_{s \sim x^{(t)}} \ell^{(t)}(s)$. We view the problem as an online convex optimization problem on the space $\Delta(S)$ of Lebesgue-continnuous distributions on $S$. We prove a general regret bound for the Dual Averaging method on $L^2(S)$, then prove that dual averaging with $\omega$-potentials (a class of strongly convex regularizers) achieves sublinear regret when $S$ is uniformly fat (a condition weaker than convexity).
\end{abstract}
%============================================================================================
\section{Introduction}

We consider an online learning problem on a compact subset $S \subset \Rbb^n$. At each iteration $t \in \Nbb$, a decision maker chooses a distribution $x^{(t)}$ on $S$, then, a loss function $\ell^{(t)} : S \to \Rbb_+$ is revealed, and the decision maker incurs loss $\Exp_{s \sim x^{(t)}}[\ell^{(t)}(s)]$. This is summarized in Problem~\ref{prob:online}.

\begin{algorithm}
\begin{algorithmic}[1]
\FOR{$t \in \Nbb$}
\STATE Decision maker chooses distribution $x^{(t)}$ over $S$.
\STATE A loss function $\ell^{(t)} : S \to \Rbb_+$ is revealed. We assume $\ell^{(t)}$ is $L$-Lipschitz.
\STATE The decision maker incurs expected loss $\Exp_{s \sim x^{(t)}}[\ell^{(t)}(s)]$
\ENDFOR
\end{algorithmic}
\floatname{algorithm}{Problem}
\caption{Online decision problem with Lipschitz losses on $S$.}
\label{prob:online}
\end{algorithm}

The regret of the decision maker is defined as follows: for a given sequence of losses $(\ell^{(t)})_{t \in \Nbb}$, and a corresponding sequence of decisions $(x^{(t)})_{t \in \Nbb}$, the cumulative regret at time~$t$, denoted by
\[
R^{(t)} = \sum_{\tau = 1}^t \Exp_{s \sim x^{(t)}}[\ell^{(t)}(s)] - \inf_{s \in S} \sum_{\tau = 1}^t \ell^{(\tau)}(s)
\]
compares the expected loss cumulated by the decision maker to the infimum of the cumulative loss function. In particular, we seek to design algorithms for which the regret grows sub-linearly in $t$, \emph{for any sequence of losses} in a given class (the assumptions on the losses will later be made explicit). 

This sequential decision problem has a long history which dates back to Hannan~\cite{hannan1957approximations} and Blackwell~\cite{blackwell1956analog}, who formulated the problem in the context of repeated games. The notion of regret is closely related to the notion of consistent play (as defined by Hannan) and approachability (as defined by Blackwell). Beyond player dynamics in repeated games, online learning has many applications such as portfolio optimization~\cite{cover91universal,blum1999universalPortfolios} and machine learning~\cite{dekel2011distributed}.

Regret minimization is essential in the design and analysis of online learning algorithms~\cite{cesa2006prediction,bubeck2012regret}, and the study of player dynamics in repeated games~\cite{hart2001general,hart2003regret,stoltz2007learning,stein2011correlated}. In this article, we study the problem of designing sublinear regret algorithms under minimal assumptions on the feasible set $S$ and the sequence of losses $(\ell^{(t)})$.

When the feasible set $S$ is finite, and the losses $(\ell^{(t)})$ are uniformly bounded, the Hedge algorithm~\cite{cesa2006prediction}, also known as the multiplicative weight updates~\cite{arora2012multiplicative} or the exponentiated gradient method~\cite{kivinen1997exponentiated}, is known to achieve sublinear regret, and is easy to analyze and to implement. More general classes of algorithms with sublinear regret have been developed since. For example, the  the online mirror descent algorithm~\cite{bubeck2012regret}, an extension of the mirror descent method due to Nemirovski and Yudin~\cite{nemirovski1983problem}, is shown to have sublinear regret for any choice of strongly convex distance-generating function. Similarly, the dual averaging method~\cite{nesterov2009primaldual} is shown to achieve sublinear regret for any choice of strongly convex regularizer. Remarkably, both of these families of algorithms include the Hedge algorithm as a special case.

When the set $S$ is infinite, designing sublinear regret algorithms requires making additional assumptions on the class of loss functions $(\ell^{(t)})$, as well as the feasible set $S$. In~\cite{zinkevich2003online}, Zinkevitch considers an online problem on a convex $S$, for convex loss functions $\ell^{(t)}$. He shows that a simple gradient descent algorithm is guaranteed to have regret which grows as $\Ocal(\sqrt{t})$. In~\cite{hazan2007logarithmic}, Hazan et al. also study the online learning problem on convex $S$, and show that for some classes of loss functions, one can achieve logarithmic regret, i.e. $R^{(t)} = \Ocal(\log t)$. In particular, they show that logarithmic regret is achieved by the Newton method when the losses are $\alpha$-strongly convex, and by the Hedge algorithm when the losses are $\alpha$-exp concave (uniformly in $t$).

{\def\arraystretch{1.5}%
\begin{table}
\centering
\small
\begin{tabular}{|l|c|c|c|}
%\hline
%Assumptions & $\ell^{(t)}$ convex & $\ell^{(t)}$ $\alpha$-exp-concave & $\ell^{(t)}$ ${L}$-Lipschitz on $v$-uniformly fat $S$ \\
\hline
Assumptions on $\ell^{(t)}$ & convex & $\alpha$-exp-concave & uniformly ${L}$-Lipschitz \\
\hline
Assumptions on $S$ & convex & convex & $v$-uniformly fat \\
\hline
Method & Gradient descent \cite{zinkevich2003online} & Hedge \cite{hazan2007logarithmic} & 
\begin{tabular}{c}
Dual Averaging with strongly convex \\
$f$-divergence s.t. $f(x) = \Ocal(x^{1+\epsilon})$ (Section~\ref{sec:omega_potential})
\end{tabular} \\
\hline
Learning rates & $t^{-\frac{1}{2}} $ & $\alpha$ & $t^{-\frac{1}{2+n\epsilon}}$
\\
\hline
$R^{(t)}/t$ & $\Ocal ( t^{-\frac{1}{2}})$ & $\Ocal \bigl(t^{-1} \log{t}\bigr)$ & $\Ocal\parenth{ t^{-\frac{1}{2+n\epsilon}} }$
\\
\hline
\end{tabular}
\caption{Regret upper bounds for different classes of losses.}
\vspace{-2ex}
\label{table:rates}
\end{table}
}

In this article, we design sublinear regret algorithms under mild assumptions on the feasible set and the sequence of losses. In particular, we only assume that the losses are Lipschitz-continuous, and relax the convexity assumption on the set $S$. Our main result is summarized in Table~\ref{table:rates}, together with regret bounds for other classes of loss functions. We show that one can formulate the online learning problem as an optimization problem over a convex subset of $L^{2}(S)$, allowing us to use results from (infinite dimensional) convex analysis. By applying the dual averaging method of Nesterov to $L^2(S)$ we prove, in Section~\ref{sec:dual_averaging}, a general regret bound which holds for any choice of regularizer. In Section~\ref{sec:omega_potential}, we consider a particular class of regularizers, which can be expressed as Csisz\'ar divergences of $\omega$-potentials, and we derive sufficient conditions on the potential to (i) make the dual averaging solution efficiently computable, and (ii) to guarantee that the regret grows sublinearly on any sequence of uniformly Lipschitz losses. This results in a general class of algorithms which are efficient to implement and which have sublinear regret guarantees under mild assumptions on the feasible set and the class of losses. In Section~\ref{sec:conclusion}, we give concluding remarks, connections with related problems, and directions for future work.

%============================================================================================
%============================================================================================
\section{Dual averaging on $L^2(S)$}
\label{sec:dual_averaging}

We start by applying Nesterov's dual averaging method~\cite{nesterov2009primaldual} to our sequential decision problem viewed as an online optimization problem on a convex subset of $L^2(S)$, and derive a general regret bound for this algorithm.

\subsection{Dual Averaging on a Hilbert space}
Consider a Hilbert space~$E$, and a feasible set $\Xcal \subset E$, assumed to be closed and convex, and let $\| \cdot \|$ be a reference norm on $E$ (not necessarily the norm induced by the inner product).

Let $\psi : \Xcal \to \Rbb_+$ be proper, continuous, and Fr\'echet-differentiable on the interior of $\Xcal$, denoted by $\mathring \Xcal$. The Bregman divergence associated to $\psi$ is defined as follows:
\[
\begin{aligned}
D_\psi : \Xcal \times \mathring \Xcal & \to \Rbb_+ \\
(x, y) & \mapsto D_\psi(x, y) = \psi(x) - \psi(y) - \braket{\nabla \psi(y)}{x - y}
\end{aligned}
\]
The function $\psi$ is said to be $\ell_\psi$-strongly convex with respect to a reference norm $\| \cdot \|$ for all $x, y \in \Xcal \times \mathring \Xcal$,
\[
D_\psi(x, y) \geq \frac{\ell_\psi}{2} \|x - y\|^2
\]
It is $L_\psi$-smooth with respect to $\| \cdot \|$ if for all $x, y \in \Xcal \times \mathring \Xcal$,
\[
D_\psi(x, y) \leq \frac{L_\psi}{2} \|x - y\|^2
\]

As we describe below, strong convexity and smoothness are dual properties. Define the Fenchel-Legendre conjugate of $\psi$
\[
\psi^*(y) = - \inf_{x \in \Xcal} \psi(x) - \braket{y}{x}
\]
Note that the minimum is attained and the minimizer is unique since $\psi$ is strongly convex and $\Xcal$ is closed and convex (Theorem~11.9 in \cite{bauschke2011convex}). The gradient of $\psi^*$ is
\[
\nabla \psi^*(y) = \argmin_{x \in \Xcal} \psi(x) - \braket{y}{x}
\]
which we will refer to as the Bregman projection onto  $\Xcal$, since it can be written as
\al{
\nabla \psi^*(y) 
= \argmin_{x \in \Xcal} \psi(x) - \braket{\nabla \psi(\nabla \psi^{-1}(y))}{x} = \argmin_{x \in \Xcal}D_\psi(x, \nabla \psi^{-1}(y))
}
%-----------------------------------------------------------------------------------------------------------------------------------------------------------------
%\begin{lemma}[\cite{bauschke2011convex}, Proposition~17.10]
%\label{lemma:convex_gradient_monotonicity}
%Let $f: \Xcal \to \Rbb$ be a Fr\'echet-differentiable function. Then
%\[
%f \text{ is convex } \eqv\forall x, y \in \Xcal, \  \braket{\nabla f(x) - \nabla f(y)}{x - y} \geq 0
%\]
%\end{lemma}
\begin{proposition}
\label{prop:smoothness_duality}
If $\psi$ is $\ell_\psi$-strongly convex with respect to $\|\cdot \|$, then $\psi^*$ is $\frac{1}{\ell_\psi}$-smooth with respect to the dual norm $\| \cdot \|_*$.
\end{proposition}
Proposition~\ref{prop:smoothness_duality} is an extension of Theorem~18.15 in~\cite{bauschke2011convex} to general norms, the proof is provided in the Appendix.

Given a sequence $(\ell^{(t)})$ of linear functionals in the dual space $E^*$, the method projects, at each step, the cumulative dual vector $- L^{(t)} = - \sum_{\tau = 1}^{t} \ell^{(\tau)}$, scaled by a step size $\eta_{t+1}$, onto the feasible set, using the Bregman projection $\nabla \psi^*$. This is  summarized in Algorithm~\ref{alg:dual_averaging}. Without loss of generality, we will assume that $\inf_{x \in \Xcal} \psi(x) = 0$.
%-----------------------------------------------------------------------------------------------------------------------------------------------------------------
\begin{algorithm}[h]
\begin{algorithmic}[1]
\FOR{$t \in \Nbb$}
\STATE Define $L^{(t)} = \sum_{\tau = 1}^{t} \ell^{(\tau)}$
\STATE Update
\begin{align}
x^{(t+1)} &= \nabla\psi^*(-\eta_{t+1} L^{(t)}) = \arg \min_{x \in \Xcal} \braket{L^{(t)}}{x} + \frac{1}{\eta_{t+1}} \psi(x)\label{eq:dual_averaging_update}
\end{align}
\ENDFOR
\end{algorithmic}
\caption{Dual averaging method with input sequence $(\ell^{(t)})$ and learning rates $(\eta_t)$}
\label{alg:dual_averaging}
\end{algorithm}
%-----------------------------------------------------------------------------------------------------------------------------------------------------------------
\subsection{Dual Averaging on $L^2(S)$}
In particular, we consider the case where $E = L^2(S)$, the Lebesgue space of square integrable functions on $S$, endowed with the inner product $\braket{f}{g} = \int_S f(s)g(s)\lambda(ds)$, where $\lambda$ is the scaled Lebesgue measure such that $\lambda(S) = 1$. Let the feasible set $\Xcal$ be
\begin{align*}
\Xcal := \bigl\{f \in L^2(S) : f \geq 0 \text{ a.e. and }\textstyle \int_S f(s) \lambda(ds) = 1\bigr\}
\end{align*}
Note that while $\Xcal$ is closed and convex, it is unbounded: if $A$ is a measurable subset of $S$, then $\frac{1}{\lambda(A)}1_{A} \in \Xcal$, and $\|\frac{1}{\lambda(A)} 1_A\|_2 = \frac{1}{\sqrt{\lambda(A)}}$, which can be arbitrarily large. 

An element $f \in \Xcal$ will be be identified with the probability distribution on $S$ with density $f$. The dual space is $E^* = L^2(S)$, and since $S$ is compact, $E^*$ contains, in particular, the set $C^0(S)$ of continuous functions on $S$. Problem~\ref{prob:online} can be viewed as follows: at each iteration $t$, the decision maker chooses an element of $\Xcal$, then an element $\ell^{(t)} \in C^0(S) \subset L^2(S)$ is revealed, and the decision maker incurs the expected loss $\braket{\ell^{(t)}}{x^{(t)}}$. Next, we define the regret and provide a first bound on the regret of the dual averaging method.

\begin{definition}
\label{def:regret}
Let $(\ell^{(t)})$ be a sequence of elements of $L^2(S)$, and consider the dual averaging algorithm on this sequence, with learning rates $(\eta_t)$. The cumulative regret of the algorithm is defined as
\[
R^{(t)} = \sup_{x \in \Xcal} \sum_{\tau = 1}^t \braket{\ell^{(\tau)}}{x^{(t)} - x} 
= \sum_{\tau = 1}^t \braket{\ell^{(\tau)}}{x^{(t)}} - \inf_{x \in \Xcal} \braket{\sum_{\tau = 1}^t \ell^{(\tau)}}{x}
\]
The regret is said to be sublinear if
$
\limsup_{t \to \infty} \frac{R^{(t)}}{t} \leq 0
$.
\end{definition}

The regret compares the cumulative loss of the algorithm, $\sum_{\tau = 1}^t \braket{\ell^{(\tau)}}{x^{(\tau)}}$, to the best cumulative loss of any stationary distribution $x$ (note that the infimum may not be attained).
%-----------------------------------------------------------------------------------------------------------------------------------------------------------------

\begin{lemma}[Dual Averaging Regret]
\label{lemma:dual_averaging_regret}
Consider the dual averaging method with dual sequence $(\ell^{(t)})$ and learning rates $(\eta^{(t)})$. Suppose that $\psi$ is $\ell_\psi$-strongly convex w.r.t. $\|\cdot\|$, and that the losses are bounded in the dual norm, uniformly in $t$, i.e. there exists $M>0$ such that for all $t$, $\|\ell^{(t)}\|_* \leq M$.  Then for all $t$ and all $x \in \Xcal$,
\al{
\sum_{\tau = 1}^t \braket{\ell^{(\tau)}}{x^{(\tau)} - x} 
&\leq\frac{1}{\eta_t}\psi(x) + \frac{M^2}{2\ell_\psi} \sum_{\tau = 1}^t \eta_\tau
}
\end{lemma}
\begin{proof}
Define the potential function
\[
\xi(\eta, L) = -\frac{1}{\eta} \psi^*(-\eta L)
\]
We first show the following inequality:
\begin{equation}
\label{eq:DA_inst_regret_bound}
\braket{x^{(t)}}{\ell^{(t)}} \leq \xi(\eta_t, L^{(t)}) - \xi(\eta_{t-1}, L^{(t-1)}) + \frac{\eta_{t}}{2\ell_\psi} \|\ell^{(t)}\|_*^2
\end{equation}
Since $\psi$ is $\ell_\psi$-strongly convex w.r.t. $\|\cdot\|$, by Proposition~\ref{prop:smoothness_duality}, $\psi^*$ is $\frac{1}{\ell_\psi}$-smooth w.r.t. $\|\cdot\|_*$, therefore $D_{\psi^*}(-\eta_{t}L^{(t)}, -\eta_{t}L^{(t-1)}) \leq \frac{1}{2\ell_\psi} \|\eta_{t}L^{(t)} -\eta_{t}L^{(t-1)}\|_*^2$, i.e.
{
\al{
\psi^*(-\eta_{t}L^{(t)}) - \psi^*(-\eta_{t}L^{(t-1)}) 
&\leq  \braket{\nabla\psi^*(-\eta_{t}L^{(t-1)})}{-\eta_{t}L^{(t)} +\eta_{t}L^{(t-1)}} + \frac{1}{2\ell_\psi}\|\eta_{t}L^{(t)}-\eta_{t}L^{(t-1)}\|_*^2 \\
&= -\eta_{t}\braket{x^{(t)}}{\ell^{(t)}} + \frac{\eta_{t}^2}{2\ell_\psi}\|\ell^{(t)}\|_*^2
}
}
Thus
\[
\braket{x^{(t)}}{\ell^{(t)}} \leq \xi(\eta_t, L^{(t)}) - \xi(\eta_t, L^{(t-1)}) + \frac{\eta_{t}}{2\ell_\psi} \|\ell^{(t)}\|_*^2
\]
It remains to show that $\eta \mapsto \xi(\eta, G)$ is decreasing. Taking the derivative with respect to $\eta$,
\al{
\partial_\eta \xi(\eta, L) 
&= \frac{1}{\eta^2} \psi^*(-\eta L) - \frac{1}{\eta} \braket{\nabla \psi^*(-\eta L)}{-L} \\
&= \frac{1}{\eta^2} \parenth{ \psi^*(-\eta L) + \braket{\nabla \psi^*(-\eta L)}{\eta L} } \\
&\leq \frac{1}{\eta^2} \psi^*(0) &\text{by convexity of $\psi^*$}\\
&= -\frac{1}{\eta^2} \inf_{x \in \Xcal} \psi(x) = 0
}
which proves inequality~\eqref{eq:DA_inst_regret_bound}. Summing, and using the bound on $\|\ell^{(t)}\|_*$, we have
\[
\sum_{\tau = 1}^t \braket{\ell^{(\tau)}}{x^{(\tau)}} 
\leq \xi(\eta_t, L^{(t)}) - \xi(\eta_0, L^{(0)}) + \frac{M^2}{2\ell_\psi} \sum_{\tau = 1}^t \eta_\tau
\]
By definition of $\xi$, we have
{\al{
&\xi(\eta_0, L^{(0)}) = \frac{1}{\eta_0} \psi^*(0) = -\frac{1}{\eta_0} \inf_{x \in \Xcal} \psi(x) = 0 \\
&\xi(\eta_t, L^{(t)}) = \frac{1}{\eta_t} \inf_{x \in \Xcal} \braket{\eta_tL^{(t)}}{x} + \psi(x) \leq \braket{L^{(t)}}{x} + \frac{1}{\eta_t}\psi(x)
}}%
Therefore
{\al{
\sum_{\tau = 1}^t \braket{\ell^{(\tau)}}{x^{(\tau)}} 
&\leq \braket{L^{(t)}}{x} + \frac{1}{\eta_t}\psi(x) + \frac{M^2}{2\ell_\psi} \sum_{\tau = 1}^t \eta_\tau
}}%
which proves the claim.
\end{proof}
%-----------------------------------------------------------------------------------------------------------------------------------------------------------------

Note that the regularizer $\psi$ can be unbounded on $S$. This is true for example for the entropy regularizer $\psi(x) = \int_{S} x(s) \ln(x(s)) \lambda(ds)$, which we will use in Section~\ref{sec:entropy}. Thus to obtain a useful (sublinear) bound on the regret, it may not suffice to take a supremum in the bound of Lemma~\ref{lemma:dual_averaging_regret}. This motivates the following Theorem. In what follows, we will assume that the loss functions are Lipschitz, uniformly in time. Let $s^\star_t \in \arg\min_{s \in S} L^{(t)}(s)$ (since the loss functions are continuous and $S$ is compact, the minimum is attained). Intuitively, if the losses are Lipschitz, then $\inf_{x \in \Xcal} \braket{L^{(t)}}{x}$ is well approximated by the cumulative loss of distributions which concentrate their mass around $s_t^\star$.

%-----------------------------------------------------------------------------------------------------------------------------------------------------------------
\begin{theorem}[Dual Averaging Regret for Lipschitz Losses]
\label{thm:dual_averaging}
Suppose that $\ell^{(t)}$ is ${L}$-Lipschitz, and $\|\ell^{(t)}\|_* \leq M$, uniformly in $t$. Then the dual averaging method with learning rates $(\eta_t)$ guarantees the following bound on the regret: For any positive sequence $(d_t)$,
\begin{align}
\frac{R^{(t)}}{t} \leq \frac{M^2}{2\ell_\psi} \frac{\sum_{\tau = 1}^t \eta_{\tau+1}}{t} + {L} d_t + \frac{1}{t \eta_{t+1}} \inf\nolimits_{x \in \Bcal_t} \psi(x)
\label{eq:dual_averaging:main_bnd}
\end{align}
where $\Bcal_t \subset \Xcal$ denotes the set of Lebesgue-continuous densities supported on $B(s^\star_t, d_t)$.
\end{theorem}
\begin{proof}
First, we observe that
\[
R^{(t)} = \sum_{\tau = 1}^t \braket{\ell^{(\tau)}}{x^{(t)}} - \inf_{x \in \Xcal} \braket{x} \leq \sum_{\tau = 1}^t \braket{\ell^{(\tau)}}{x^{(\tau)}} - L^{(t)}(s^\star_t)
\]
Since the losses are ${L}$-Lipschitz, we have ${\forall x \in \Bcal_t}$,
\begin{align*}
\braket{\sum_{\tau = 1}^t \ell^{(\tau)}}{x} 
%&= \int_{B(s^\star_t, d_t)} \textstyle \sum_{\tau = 1}^t \ell^{(\tau)}(s)x(s)\,ds \\
%&\leq \int_{B(s^\star_t, d_t)} \textstyle \sum_{\tau = 1}^t (\ell^{(\tau)}(s^\star_t) + {L}d_t)x(s)\,ds \\
= \int_{B_t} \sum_{\tau = 1}^t \ell^{(\tau)}(s)x(s)\,ds 
\leq \int_{B_t} \sum_{\tau = 1}^t (\ell^{(\tau)}(s^\star_t) + {L}d_t)x(s)\,ds 
= L^{(t)}(s^\star_t) + t\, {L} d_t
\end{align*}
Thus, for all $x \in \Bcal_t$,
\al{
R^{(t)} 
&\leq \sum_{\tau = 1}^t \braket{\ell^{(\tau)}}{x^{(\tau)}} - L^{(t)}(s^\star_t) \\
&\leq \sum_{\tau = 1}^t \braket{\ell^{(\tau)}}{x^{(\tau)} - x} + t {L} d_t \\
&\leq \frac{1}{\eta_{t+1}} \psi(x) + \frac{M^2}{2\ell_\psi} \sum_{\tau = 1}^{t} \eta_{\tau+1} + t {L}d_t
}
where the last inequality uses Lemma~\ref{lemma:dual_averaging_regret}. We conclude by dividing by $t$ and taking the infimum over $x \in \Bcal_t$.
\end{proof}

We now have a general regret bound for the dual averaging method applied to Problem~\ref{prob:online}. In the next section, we further study the dual averaging algorithm with a particular family of regularizers, and we study their properties.
%-----------------------------------------------------------------------------------------------------------------------------------------------------------------

%============================================================================================
\section{Dual Averaging with $\omega$-potentials on Uniformly Fat Sets}
\label{sec:omega_potential}
We now study the dual averaging method when the regularizer is the $f$-divergence, or Csisz\'ar divergence~\cite{csiszar1967information} of a particular class of potential functions. This definition is a generalization of~\cite{audibert2014regret} to our infinite dimensional Hilbert setting.

\subsection{Csisz\'ar divergence induced by $\omega$-potentials}
\begin{definition}
\label{def:potential}
Let $\omega \leq 0$ and $a \in (-\infty, +\infty]$. An increasing $C^1$ diffeomorphism $\phi: (-\infty, a) \to (\omega, \infty)$, is an $\omega$-potential if
\al{
\lim_{u \to -\infty} \phi(u) = \omega && \lim_{u \to a} \phi(u) = +\infty && \int_{0}^1 \phi^{-1}(u)du < \infty
}
We associate to $\phi$ the function $f_\phi$, defined on $(0, \infty)$,
\[
f_\phi(x) = \int_{1}^x \phi^{-1}(u)du
\]
which is, by definition, convex (since $\phi^{-1}$ is increasing), and satisfies $f_\phi(1) = 0$\footnote{Note that in the original definition of Audibert at al., the function $f_\phi$ is taken to be the integral from $\omega$ to $\omega + x$. Our definition corresponds to a translation of $f_\phi$, in order to have $f_\phi(1) = 0$, which guarantees that $\psi_{f_\phi}(x) \geq 0$ on $\Xcal$.}. We also associate the $f$-divergence, defined on $\Xcal$ by
\[
\psi_{f_\phi}(x) = \int_{S} f_\phi(x(s) ) \lambda(ds)
\]
By convexity of $f$, we have for all $x \in \Xcal$, $\psi_{f_\phi}(x) \geq f_\phi \parenth{\int_S x(s) ds} = f_\phi(1) = 0$.
\end{definition}
\begin{figure}[h]
\centering
\includegraphics[width=.3\textwidth]{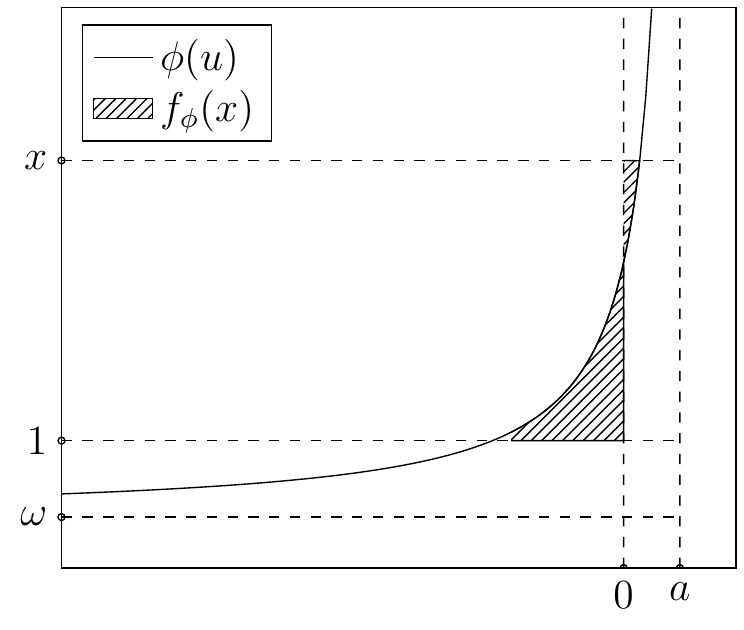}
\caption{Illustration of an $\omega$-potential.}
\end{figure}
\begin{example}[Euclidean projection]
Perhaps the simplest instance of $\omega$-potential is the identity $\phi(u) = u$, for which $\omega = -\infty$ and $a = +\infty$. In this case, $f_\phi(x) = \frac{x^2 - 1}{2}$, and the resulting Csisz\'ar divergence is $\psi_{f_\phi}(x) = \frac{1}{2}\int_S x(s)^2 \lambda(ds) - \frac{1}{2}= \frac{\|x\|^2_2 - 1}{2}$. The dual averaging method then projects, at each iteration, in the $L^2$ norm, the dual vector $-\eta_{t+1}\sum_{\tau = 1}^t \ell^{(t)}$ on the feasible set.
\end{example}

\begin{example}
\label{example:p-norm}
More generally, if $p > 1$, then taking $\phi(u) = \sign(u) |u|^{\frac{1}{p-1}}$ is an $\omega$-potential with $\omega = -\infty$ and $a = +\infty$. In this case, $\phi^{-1}(u) = \sign(u) |u|^{p-1}$ and $f_\phi(x) = \frac{|x|^p}{p}$, and the corresponding Csisz\'ar divergence is $\psi_{f_\phi}(x) = \frac{\|x\|_p^p - 1}{p}$.
\end{example}

\begin{example}[Entropy projection]
If we take $\phi(u) = e^{u-1}$, then the corresponding density function is $f_\phi(x) = \int_1^x \phi^{-1}(u)du = \int_{1}^x (1+\ln u)du = x \ln x$, and the associated $f_\phi$ divergence is the negative entropy
\al{
\psi_{f_\phi}(x) &= \int_{S} x(s) \ln x(s) \lambda(ds)
}
See Section~\ref{sec:entropy} for a generalization of the entropy divergence.
\end{example}

%-----------------------------------------------------------------------------------------------------------------------------------------------------------------
\subsection{Strong convexity}
In order for the bound of Theorem~\ref{thm:dual_averaging} to hold, we need the regularizer to be strongly convex. In this section, we give sufficient conditions on the potential for strong convexity of~$\psi_{f_\phi}$ with respect to $p$ norms, defined as follows for $p \geq 1$
\[
\|x\|_{p} = \parenth{\int_S |x(s)|^p \lambda(ds)}^{\frac{1}{p}}.
\]
\begin{theorem}
\label{thm:DA_strong_convexity}
Let $\phi$ be an $\omega$-potential, and suppose that there exists $\alpha > 0$, $z_0 \geq 0$ and $r \in (0, 1]$ such that $(\phi^{-1})'(z) \geq\frac{1}{ \alpha (z + z_0)^{r}}$ for all $z > 0$. Then $f_\phi$ is $\ell$-strongly convex w.r.t. the $p$ norm, with $\ell = \frac{1}{\alpha(1+z_0)^r}$ and $p = \frac{2}{1+r}$. That is, for all $x, y \in \Xcal$,
\[
D_{f_\phi}(x, y) \geq \frac{1}{2\alpha(1+z_0)^r} \|x - y\|_{\frac{2}{1+r}}^2
\]
\end{theorem}

\begin{proof}
By definition, the $f$-divergence associated to the potential $\phi$ is differentiable at any $x \in \Xcal$ with $x > 0$ a.e., and has gradient
\[
\nabla \psi_{f_\phi}(x)(s) = \phi^{-1}(x(s))
\]
Thus,
{\al{
D_{f_\phi}(x, y) &= \psi(x) - \psi(y) - \braket{\nabla \psi(y)}{x - y} \\
&= \int_S f_\phi(x(s)) - f_\phi(y(s)) - \phi^{-1}(y(s))(x(s) - y(s)) \lambda(ds) \\
&= \int_S \int_{y(s)}^{x(s)} \phi^{-1}(u)du - \phi^{-1}(y(s))(x(s) - y(s)) \lambda(ds)
}}%
and by a Taylor expansion of $\phi^{-1}$, there exists $z \in [x, y]$ %(that is, there exists $t \in [0, 1]$ with $z = tx + (1-t)y$)
such that
\al{
\int_{y(s)}^{x(s)} \phi^{-1}(u)du - \phi^{-1}(y(s))(x(s) - y(s)) 
&= (\phi^{-1})'(z(s))\frac{(x(s) - y(s))^2}{2} \\
&\geq \frac{1}{2}\frac{(x(s) - y(s))^2}{\alpha (z(s)+z_0)^{r}} \hspace{2cm} &\text{by assumption on $\phi$}
}
Now by the Cauchy-Schwartz inequality, if $p, q \geq 1$ are conjugate, i.e. $\frac{1}{p} + \frac{1}{q} = 1$, then for any $(a, b) \in \Hcal^* \times \Hcal$, $\braket{a}{b} \leq \|a\|_p \|b\|_q$, and it follows that whenever $\|b\|_q > 0$, $\|a\|_p^p \geq \parenth{\frac{\braket{a}{b}}{\|b\|_q}}^{p}$. Applying this inequality with $a = \parenth{\frac{(x - y)^2}{\alpha (z + z_0)^{r}}}^{\frac{1}{p}}$, $b = \parenth{\alpha (z+z_0)^{r}}^{\frac{1}{p}}$, we have
\[
\int_S \frac{(x(s) - y(s))^2}{\alpha (z(s)+z_0)^{r}} \lambda(ds) 
\geq \parenth{\frac{\int_S (x(s) - y(s))^{\frac{2}{p}} \lambda(ds) }{ \parenth{ \int_S \parenth{\alpha (z(s)+z_0)^{r}}^{\frac{q}{p}} \lambda(ds)  } ^{\frac{1}{q}} } }^p 
= \frac{\|x - y\|_{\frac{2}{p}}^2}{\alpha \|z + z_0\|_{\frac{r q}{p}} ^{r}}
\]
In particular, if we take $p = 1+r, q = 1+ \frac{1}{r}$, then $\frac{r q}{p} = 1$, and $\|z+z_0\|_{\frac{r q}{p}} = \|z+z_0\|_1 = 1 + z_0$ since $z \in [x, y] \subset \Xcal$ and $z_0 \geq 0$. Therefore
\al{
D_{f_\phi}(x, y) 
\geq \frac{1}{2} \int_S \frac{(x(s) - y(s))^2}{\alpha (z(s)+z_0)^{r}} \lambda(ds) 
\geq \frac{1}{2} \frac{\|x - y\|_{\frac{2}{1+r}}^2}{\alpha(1+z_0)^r}
}
which concludes the proof.
\end{proof}
%-----------------------------------------------------------------------------------------------------------------------------------------------------------------

As a consequence of Theorem~\ref{thm:DA_strong_convexity}, we can show that Begman divergences of Example~\ref{example:p-norm}, $\psi_{f} = \frac{\|\cdot\|_p^p - 1}{p}$, $p \in (1, 2]$, are strongly convex w.r.t. $\|\cdot\|_p$.
%-----------------------------------------------------------------------------------------------------------------------------------------------------------------
\begin{corollary}
\label{corollary:p-norm-strong-convexity}
Let $p \in (1, 2)$, and consider the $\omega$-potential $\phi(u) = \sign(u) |u|^{\frac{1}{p-1}}$, and its corresponding Csisz\'ar divergence $\psi(x) = \frac{\|x\|_p^p - 1}{p}$. Then $\psi$ is $(p-1)$-strongly convex w.r.t. $\|\cdot\|_{\frac{2}{3-p}}$.
\end{corollary}
\begin{proof}
We have $\phi^{-1}(u) = \sign(u)|u|^{p-1}$, thus for $z > 0$,
\al{
(\phi^{-1})'(z) = \frac{p-1}{z^{2-p}}
}
where $2 - p \in (0, 1)$ by assumption on $p$. Thus we can apply Theorem~\ref{thm:DA_strong_convexity} with $r = 2-p$, $\alpha = \frac{1}{p-1}$ and $z_0 = 0$, which proves the claim.
\end{proof}

Note that the corollary also holds for $p=2$, since one can explicitly compute the Bregman divergence: we have $\psi(x) = \frac{\|x\|^2_2 - 1}{2}$, thus
\[
D_\psi(x, y) = \psi(x) - \psi(y) - \braket{\nabla \psi(y)}{x - y} = \frac{1}{2} (\|x\|_2^2 - \|y\|^2_2 - 2\braket{y}{x - y}) = \frac{1}{2} \|x - y\|_2^2
\]
which is $1$-strongly convex w.r.t. $\|\cdot\|_2$. Finally, we observe that a similar result is proved, in the finite-dimensional case, in~\cite{bental2001oredered}, Lemma 8.1: for $p \in (1, 2]$, $\frac{1}{2}\|\cdot\|_p^2$ is $(p-1)$-strongly convex w.r.t. $\|\cdot\|_p$. Note that the result concerns $\|\cdot\|_p^2$ while Corollary~\ref{corollary:p-norm-strong-convexity} concerns $\|\cdot\|_p^p$. The squared $p$-norm is not a Csisz\'ar divergence induced by an $\omega$-potential in general (except for $p = 2$). Using $\|\cdot\|_p^p$ as a regularizer instead of $\|\cdot\|_p^2$ allows us to benefit from the properties of $\omega$-potentials; in particular, the solution of the dual averaging iteration can be computed efficiently, as discussed in Section~\ref{sec:solution}.

%-----------------------------------------------------------------------------------------------------------------------------------------------------------------
Next, we give another sufficient condition for strong convexity w.r.t. $\|\cdot\|_1$.
\begin{proposition}
Let $\phi$ be a $\omega$-potential, and assume that $\phi$ is a $C^2$-diffeomorphism. Consider the potential density $f_\phi$ as in Definition~\ref{def:potential}. Let $w = 1 +\frac{1}{3} \frac{f_\phi'''(1)}{f_\phi''(1)}$. If $f_\phi$ satisfies one of the following conditions $\forall u > 0$:
{\begin{align*}
&(f_\phi(u) - f_\phi'(1)(u-1))\parenth{1 + (1-w)(u-1)} \geq \frac{f_\phi''(1)}{2}(u-1)^2\\
&sgn(u-1)\parenth{\frac{f_\phi'''(u)}{f_\phi''(u)}[1+(1-w)(u-1)] + 3(1-w)} \geq 0,
\end{align*}
}%
then the $f_\phi$-divergence is strongly convex with respect to the total variation norm. More precisely, for all $x, y \in \Xcal$,
\[
D_{f_\phi}(x, y) \geq\frac{1}{8\phi'(\phi^{-1}(1))} \|x - y\|_{1}^2
\]
\end{proposition}
\begin{proof}
By Definition~\ref{def:potential}, if $\phi$ is a $C^2$ diffeomorphism, then $f_\phi$ is three times differentiable on $(0, \infty)$, and for all $x > 0$, $f''(x) = (\phi^{-1})'(x) = \frac{1}{\phi'(\phi^{-1}(x))}$, which is, by assumption on $\phi$, strictly positive. Thus by the generalized Pinsker inequality in~\cite{gilardoni2010pinsker} (Theorem~3 and Corollary~4) the associated $f$-divergence satisfies $D_{f_\phi}(x, y) \geq\frac{f''(1)}{2} D_{TV}(x, y)^2$, where $D_{TV}$ is the total variation norm, $D_{TV}(x, y) = \frac{1}{2}\|x - y\|_1$. This concludes the proof.
\end{proof}

%-----------------------------------------------------------------------------------------------------------------------------------------------------------------
\subsection{Solution of the Bregman projection with $\omega$-potentials}
\label{sec:solution}
We now characterize the solution of the dual averaging update, given by the Bregman projection in equation~\eqref{eq:dual_averaging_update}.

\begin{proposition}
\label{prop:DA_solution}
Let $\phi$ be an $\omega$-potential. Let $L^{(t)} \in E^*$, and consider the dual averaging iteration~\eqref{eq:dual_averaging_update} in Algorithm~\ref{alg:dual_averaging}, with the regularizer $\psi$ taken to be the $f_\phi$ divergence associated to $\phi$. Then the solution $x^{(t+1)}$ is given by
\begin{align*}
x^{(t+1)}(s) = \phi(-\eta_{t+1}(L^{(t)}(s) + \nu^\star))_+
\end{align*}
where $y_+$ denotes the positive part of $y$, and $\nu^\star$ satisfies $\int_S \phi(-\eta_{t+1}(L^{(t)}(s) + \nu^\star))_+ \lambda(ds) = 1$.
\end{proposition}

\begin{proof}
Let $K$ be the cone $K = \{x \in L^2(S) : x \geq 0 \ a.e.\}$, and let
\[
f(x) = \braket{L^{(t)}}{x} + \frac{1}{\eta_{t+1}} \psi(x) + i_K(x)
\]
where $i_K$ is the indicator function of the cone $K$, i.e. $i_K(s) = 0$ if $s \in K$ and $+\infty$ otherwise. The dual averaging iteration is equivalent to the following problem:
\[
\begin{aligned}
&\text{minimize}_{x \in L^2(S)} && f(x)\\
&\text{subject to} && \braket{\mathbf 1}{x} = 1
\end{aligned}
\]
where $ 1: S \to \Rbb$ is identically equal to $1$. Using the fact that the subdifferential of the indicator $i_K$ is the normal cone $N_K$ (See for example Chapter 16 in~\cite{bauschke2011convex}) given by $\forall x \in K$
\[
\partial i_K(x) = N_K(x) = \Bigl\{ g \in L^2(S) : \sup_{y \in K} \braket{g}{y - x} \leq 0 \Bigr\},
\]
the subdifferential of the objective function is
\[
\partial f(x) = L^{(t)} + \frac{1}{\eta_{t+1}} \phi^{-1}(x) + N_K(x)
\]
First, we show that, for all $x$ and all $g \in N_K(x)$, $g \leq 0$ and $g x = 0$ almost everywhere. Indeed, fixing $x\in K$, we have $\braket{g}{y - x} \leq 0$ for all $y \in K$. In particular, $x + 1_{g > 0} \in K$, thus
\al{
0 \geq \braket{g}{1_{g>0}} = \int_{S} g_+(s)\lambda(ds)
}
which proves that $g_+ = 0$ a.e.. Furthermore, taking $y = \frac{1}{2}x \in K$, we have
\al{
0 \geq \braket{g}{y - x} 
= -\frac{1}{2} \braket{g}{x} = \frac{1}{2} \int_S |g(s)| x(s) ds
}
which implies that $|g|x = 0$ a.e., which proves the claim.

Now, consider the Lagrangian $\Lcal: E \times \Rbb \to \Rbb$
\[
\Lcal(x, \nu) = \braket{L^{(t)}}{x} + \frac{1}{\eta_{t+1}} \psi(x) + i_K(x) + \nu (\braket{\mathbf  1}{x} - 1)
\]
Then  $(x^\star, \nu^\star)$ is an optimal pair only if
\al{
&0 \in L^{(t)} + \frac{1}{\eta_{t+1}} \phi^{-1}(x^\star) + N_K(x^\star) + \nu \mathbf 1\\
&\braket{\mathbf 1}{x^\star} = 1
}
see for example Section 19.3 in~\cite{bauschke2011convex}. We can rewrite the stationarity condition as follows:
\begin{align*}
\exists\, g^\star \in N_K(x^\star) \text{ such that } L^{(t)} + \frac{1}{\eta_{t+1}} \phi^{-1}(x^\star) + \nu \mathbf 1 + g^\star = 0.
\end{align*}
Therefore,
\al{
&x^\star(s) = \phi(-\eta_{t+1} (L^{(t)}(s)+\nu^\star + g^\star(s)) \text{ a.e.}\\
&g^\star \in N_K(x^\star) \\
&\braket{\mathbf 1}{x^\star} = 1
}
In particular, let $\support(x^\star) = \{s : x^\star(s) > 0\}$ (a measurable set). By the complementary slackness condition, $g^\star x^\star = 0$ a.e., therefore $g^\star = 0$ a.e. on $\support(x^\star)$. And for a.e. $s \notin \support(x^\star)$, we have
{\al{
\phi(-\eta_{t+1}(L^{(t)}(s) + \nu^\star)
\leq \phi(-\eta_{t+1}(L^{(t)}(s) + \nu^\star + g^\star(s)))
= 0
}}%
since $\phi$ is increasing and $g^\star \leq 0$ a.e.. Therefore the optimality conditions become
\al{
&x^\star(s) = \phi(-\eta_{t+1} (L^{(t)}(s)+\nu^\star))_+ \\
&\int_S \phi(-\eta_{t+1} (L^{(t)}(s)+\nu^\star))_+ds = 1
}
which proves the claim.
\end{proof}
Proposition~\ref{prop:DA_solution} shows that the solution of the Bregman projection~\ref{eq:dual_averaging_update} is entirely determined by the dual variable $\nu^\star$, therefore computing the solution reduces to computing the optimal $\nu^\star$. Furthermore, we observe that the function $\nu \mapsto \int_S \phi(-\eta_{t+1}(L^{(t)}(s) + \nu^\star))_+ \lambda(ds)$ is increasing, by assumption on $\phi$, therefore one can compute $\nu^\star$ (to arbitrary precision) using a simple bisection method. Note that in general, the solution $x^{(t+1)}$ may not be supported everywhere on $S$, unless $\omega = 0$, in which case $\phi$ is by definition, strictly positive.
\subsection{Regret analysis}
Next, we show that under the appropriate assumptions on the feasible set $S$, and the asymptotic behavior of the $\omega$-potential, it is possible to achieve sublinear regret with dual averaging. First, we focus our attention on sets which are uniformly fat~(using the definition of \cite{krichene2015hedge}), a generalization of convexity.
\begin{definition}
Consider a subset $S \subset \Rbb^n$. $S$ is said to be $v$-uniformly fat if there exists $v > 0$ such that for all $s \in S$, there exists a convex $K_s \subset S$ such that $s \in K_s$ and ${\lambda(K_s) \geq v}$.
\end{definition}
In particular, if $S$ is convex, it is $1$-uniformly fat. 
%-----------------------------------------------------------------------------------------------------------------------------------------------------------------
\begin{figure}[H]
\centering
\includegraphics[width=.18\textwidth]{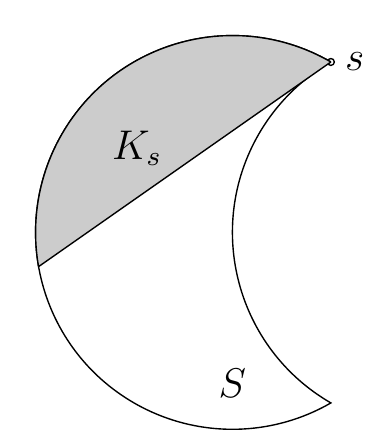} \hspace{1in}
\includegraphics[width=.18\textwidth, page=2]{TikZ/fatness}
\caption{Illustration of uniform fatness (left) and the homothetic transformation used in the proof of Proposition~\ref{prop:f-div-regret} (right).}
\label{fig:fatness}
\end{figure}

Intuitively, the uniform fatness condition guarantees that there is sufficient volume around any point of $S$, so that the solution $x^{(t)}$ of the Bregman projection assigns enough probability mass around the optimum. In particular, uniform fatness excludes isolated points.

The next Proposition gives a regret bound on uniformly fat sets. For a subset $C \subset \Rbb^n$, we denote the diameter of $C$ by $D(C) = \sup_{s, s' \in C} \|s - s'\|$.
%-----------------------------------------------------------------------------------------------------------------------------------------------------------------
\begin{proposition}
\label{prop:f-div-regret}
Suppose that $S$ is $v$-uniformly fat. Let $\phi$ be an $\omega$-potential, and take the regularizer to be the Csisz\'ar divergence  $f_\phi$. Suppose that $f_\phi$ is $\ell_\psi$-strongly convex w.r.t. $\|\cdot \|$, and that $\|\ell^{(t)}\|_*$ is bounded by $M$. Let $(d_t)$ be a sequence of positive numbers. Then for all $t$,
\begin{align}
\frac{R^{(t)}}{t} \leq \frac{M^2}{2\ell_\psi} \frac{\sum_{\tau = 1}^t \eta_{\tau+1}}{t} + {L} D(S) d_t + \frac{1}{t \eta_{t+1}} vd_t^nf_\phi\parenth{\frac{1}{vd_t^n}}
\end{align}
\end{proposition}
\begin{proof}
Let $s^\star_t \in \argmin_{s \in S} L^{(t)}(s)$. Since $S$ is $v$-uniformly fat, there exists  a convex subset $K_t \subset S$, containing $s^\star_t$, such that $\lambda (K_t) \geq v$. Following the argument in~\cite{krichene2015hedge} and~\cite{hazan2007logarithmic}, consider the homothetic transformation of $K_t$, given by
\[
S_t = \{d_t(s - s^\star_t), s \in S\}
\]
The diameter and mass of $S_t$ satisfy
\al{
&D(S_t) = d_t D(K_t) \leq d_t D(S) \\
&\lambda(S_t) = d_t^n \lambda(K_t) \geq v d_t^n
}
Let $u_{S_t}$ be the uniform distribution over $S_t$. Since $u_{S_t}$ has support in $B(s^\star_t, d_tD(S))$, we have, by Theorem~\ref{thm:dual_averaging},
\[
\frac{R^{(t)}}{t} \leq \frac{M^2}{2\ell_\psi} \frac{\sum_{\tau = 1}^t \eta_{\tau+1}}{t} + {L} D(S) d_t + \frac{1}{t \eta_{t+1}} \psi(u_{S_t})
\]
where, by definition of the $f_\phi$ divergence,
\al{
\psi(u_{S_t}) 
= \int_{S_t} f_\phi(1 / \lambda(S_t)) \lambda(ds) = f_\phi(1 / \lambda(S_t)) \lambda(S_t)
}
To conclude, we observe that $f_\phi$ is increasing on $[1, \infty)$. Indeed, by definition of an $\omega$-potential, $f_\phi(x)$ has derivative $\phi^{-1}(x)$, and since $\phi^{-1}$ is increasing, for all $x \geq 1$, $\phi^{-1}(x) \geq \frac{1}{1-\omega}\int_{\omega}^1 \phi^{-1}(u)du = 0$. Therefore
\[
\psi(u_{S_t}) = f_\phi(1 / \lambda(S_t)) \lambda(S_t) \geq vd_t^nf_\phi\parenth{\frac{1}{vd_t^n}}
\]
which concludes the proof. \end{proof}

%-----------------------------------------------------------------------------------------------------------------------------------------------------------------
Next, we show that if $f_\phi$ has a bounded asymptotic growth rate, the regret grows sublinearly.
\begin{theorem}
\label{thm:f-div-sublinear-regret}
Suppose that $S$ is $v$-uniformly fat. Let $\phi$ be an $\omega$-potential such that $f_\phi$ is $\ell_\psi$-strongly convex with respect to $\|\cdot\|$, and suppose that $\ell^{(t)}$ is $L$-Lipschitz and $\|\ell^{(t)}\|_* \leq M$ uniformly in $t$. Suppose that there exists $\epsilon > 0$ and $C > 0$ such that
\[
f_\phi(x) \leq C x^{1+\epsilon}
\]
for all $x \geq 1$. Then the dual-averaging method with $f_\phi$-divergence and learning rates $\eta_t = \theta t^{-\alpha}$ satisfies the following bound on the per-round regret
\al{
\frac{R^{(t)}}{t} 
&\leq \frac{M^2 \theta}{2\ell_\psi(1-\alpha)} t^{-\alpha} + (LD(S) + Cv^\epsilon) t^{-\frac{1-\alpha}{1+n\epsilon}} \\
&= \Ocal\parenth{t^{-\alpha} + t^{-\frac{1-\alpha}{1+n\epsilon}}}
}
The rate is optimal for $\alpha = \frac{1}{2+n\epsilon}$ for which
\[
\frac{R^{(t)}}{t} = \Ocal\parenth{t^{-\frac{1}{2+n\epsilon}}}
\]
\end{theorem}

\begin{proof}
Let $(d_t)$ be a decreasing sequence which converges to $0$. By Proposition~\ref{prop:f-div-regret}, the regret of the dual averaging method is bounded by
\[
\frac{R^{(t)}}{t} \leq \frac{M^2}{2\ell_\psi} \frac{\sum_{\tau = 1}^t \eta_\tau}{t} + {L} D(S) d_t + \frac{1}{t\eta_t} vd_t^nf_\phi\parenth{\frac{1}{vd_t^n}}
\]
where we can bound $\sum_{\tau = 1}^t \eta_\tau \leq \theta \sum_{\tau = 1}^t \int_{\tau - 1}^\tau u^{-\alpha}du = \theta \int_0^t u^{-\alpha}du = \frac{\theta}{1-\alpha} t^{1-\alpha}$. By assumption on $f_\phi$, we have
\[
vd_t^n f_\phi\parenth{\frac{1}{vd_t^n}} \leq C v^{-\epsilon} d_t^{-n\epsilon}
\]
Combining these bounds, we have
\al{
\frac{R^{(t)}}{t} 
&\leq \frac{M^2\theta}{2\ell_\psi(1-\alpha)} t^{-\alpha} + LD(S)d_t + C v^\epsilon t^{\alpha-1} d_t^{-n\epsilon} \\
&= \Ocal\parenth{t^{-\alpha} + d_t + t^{\alpha - 1} d_t^{-n\epsilon}}
}
Taking $d_t = t^{-\frac{1-\alpha}{1+n\epsilon}}$, the bound becomes $\frac{R^{(t)}}{t} \leq \frac{M^2 \theta}{2\ell_\psi(1-\alpha)} t^{-\alpha} + (LD(S) + Cv^{-\epsilon}) t^{-\frac{1-\alpha}{1+n\epsilon}} $, which proves the claim.
\end{proof}

One can formulate similar regret bounds under different assumptions on the asymptotic behavior of $f_\phi$. For example, one can show the following extension, proved in the Appendix.
\begin{customthm}{\ref{thm:f-div-sublinear-regret}.1}
Under the assumptions of Theorem~\ref{thm:f-div-sublinear-regret}, suppose that there exists $\epsilon > 0$, $\nu \geq 0$ such that
\[
f_\phi(x) = \Ocal\parenth{ x^{1+\epsilon} (\ln x)^\nu }
\]
and the learning rates are taken to be $\eta_t = \Theta\parenth{(\ln t)^{\frac{\nu}{2}}t^{-\alpha}}$, $\alpha = \frac{1}{2+n\epsilon}$ then
\[
\frac{R^{(t)}}{t} = \Ocal\parenth{ (\ln t)^{\frac{\nu}{2}}t^{-\frac{1}{2+n\epsilon}} }.
\]
\end{customthm}

To conclude this Section, we observe that while the regret is defined with respect to elements of $\Xcal$ (Definition~\ref{def:regret}), it is equivalent, for uniformly fat sets, to the regret with respect to elements of $S$, in the following sense:
\begin{align}
R^{(t)} 
&= \sum_{\tau = 1}^t \braket{\ell^{(\tau)}}{x^{(\tau)}} - \inf_{x \in \Xcal}  \braket{ L^{(t)}}{x} \notag \\
&= \sum_{\tau = 1}^t \braket{\ell^{(\tau)}}{x^{(\tau)}} - \min_{s \in S} L^{(t)}(s)
\label{eq:cumreg:dual_rewritten}
\end{align}

\begin{proof}
Let $s^\star_t \in \argmin_{s \in S} L^{(t)}(s)$. Then it suffices to show that for all $\epsilon > 0$, there exists $x \in \Xcal$ such that
\[
\braket{\sum_{\tau = 1}^t \ell^{(\tau)}}{x} \leq \sum_{\tau = 1}^t \ell^{(\tau)}(s^\star_t)+ \epsilon
\]
Fix $\epsilon > 0$. Since $S$ is $v$-uniformly fat, there exists a convex set $K_t \subset S$ containing $s^\star_t$, with $\lambda(K_t) > 0$. Let $\bar S_t$ be the homothetic transform of $K_t$, of center $s^\star_t$ and ratio $d_t$, as in the proof of Proposition~\ref{prop:f-div-regret}. Then we have
\al{
&D(\bar S_t) = d_t D(K_t) \leq d_t D(S) \\
&\lambda(\bar S_t) = d_t^n \lambda(K_t) > 0
}
Now consider  $x = \frac{1}{\lambda(\bar S_t)} 1_{\bar S_t}$. We have $x \in \Xcal$, and since the $\ell^{(\tau)}$ are uniformly ${L}$-Lipschitz,
\al{
\braket{\sum_{\tau = 1}^t \ell^{(\tau)}}{x} 
&= \sum_{\tau = 1}^t \int_{\bar S_t} \frac{1}{\lambda(\bar S_t)}\ell^{(\tau)}(s) ds\\
&\leq \sum_{\tau = 1}^t \int_{\bar S_t} \frac{1}{\lambda(\bar S_t)}(\ell^{(\tau)}(s^\star_t) + {L} d_t D(S)) ds  \\
&= t{L} d_t D(S) + \sum_{\tau = 1}^t \ell^{(\tau)}(s^\star_t)
}
In particular, if we choose $d_t = \frac{\epsilon}{t {L} D(S)}$, we have $\braket{\sum_{\tau = 1}^t \ell^{(\tau)}}{x} \leq  \sum_{\tau = 1}^t \ell^{(\tau)}(s^\star_t) + \epsilon$, which proves the claim.
\end{proof}

%-----------------------------------------------------------------------------------------------------------------------------------------------------------------
\section{Entropy dual averaging}
\label{sec:entropy}
Let $\omega \leq 0$, and consider the $\omega$-potential
\[\begin{aligned}
\phi:(-\infty, \infty) &\to (\omega, \infty) \\
u &\mapsto \phi(u) = e^{u-1} + \omega
\end{aligned}
\]
Its inverse is $\phi^{-1}(u) = 1+\ln(u - \omega)$. In particular, we have $\int_{0}^1 \phi^{-1}(u) du = (u-\omega) \ln (u-\omega) |_0^1 < \infty$. The resulting density function is simply
\[
f_\phi(x) = \int_1^x \phi^{-1}(u)du = (x-\omega) \ln (x-\omega) - (1-\omega)\ln(1-\omega)
\]
and the associated $f_\phi$ divergence is the following generalized negative entropy
\al{
\psi_{f_\phi}(x) &= \int_{S} (x(s)-\omega) \ln (x(s)-\omega) \lambda(ds) - (1-\omega)\ln(1-\omega) \\
&=H(x-\bomega) - H(\bone-\bomega)
}
where $H(x) = \int_S x(s)\ln x(s) \lambda(ds)$. By applying Proposition~\ref{prop:DA_solution}, we can derive the explicit solution of the dual averaging update~\eqref{eq:dual_averaging_update}.

%-----------------------------------------------------------------------------------------------------------------------------------------------------------------
\begin{figure}[h]
\centering
\includegraphics[page=2,width=.3\textwidth]{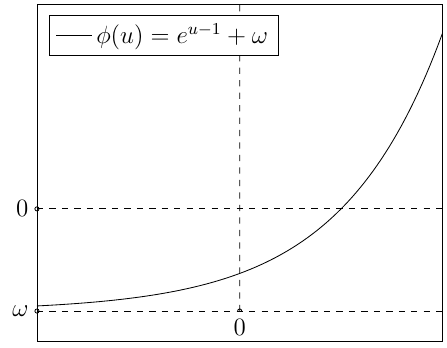}
\caption{Illustration of the entropy $\omega$-potential $\phi(u) = e^{u-1} + \omega$.}
\label{fig:potentials}
\end{figure}
%-----------------------------------------------------------------------------------------------------------------------------------------------------------------

\begin{corollary}
Consider the dual averaging method on the set of distributions $\Xcal$, regularized by the negative entropy. The solution of update~\eqref{eq:dual_averaging_update} is given by
\[
x^{(t+1)}(s) = \parenth{\frac{e^{-\eta_{t+1} L^{(t)}(s)}}{Z^{(t)}} + \omega}_+
\]
where $Z^{(t)}$ is the appropriate normalization constant. In particular, when $\omega = 0$, $x^{(t+1)}(s) = \frac{e^{-\eta_{t+1} L^{(t)}(s)}}{Z^{(t)}}$ and $Z^{(t)} =\int_S e^{-\eta_{t+1} L^{(t)}(s)} \lambda(ds)$, we recover the Hedge algorithm~\cite{hazan2007logarithmic}.
\end{corollary}
%This results in a generalization of the finite dimensional case, in which Hedge algorithm is known to be an instance of the dual averaging method, when the feasible set $\Xcal$ is a probability simplex, and the distance generating function $\psi$ is the negative entropy function, see for example~\cite{nesterov2009primaldual}, \cite{beck2003mirror}.

By definition of the entropy potential function, we have $(\phi^{-1})'(z) = \frac{1}{z - \omega}$, therefore, the assumptions of Theorem~\ref{thm:DA_strong_convexity} hold with $\alpha = 1$, $z_0 = -\omega$, $r = 1$, thus the negative entropy is $\frac{1}{1-\omega}$-strongly convex with respect to $\|\cdot\|_1$, and we can apply Proposition~\ref{prop:f-div-regret} and Theorem~\ref{thm:f-div-sublinear-regret} to obtain a regret bound.
\begin{corollary}
Suppose that $S$ is $v-$uniformly fat, and that the loss functions are $L$-Lipschitz, uniformly in time, and that $\|\ell^{(t)}\|_\infty \leq M$. Then the dual averaging method on $\Xcal$, regularized with the negative entropy, with learning rates $\eta_{t} = \Theta\parenth{\sqrt{ \frac{\ln t}{t}}}$, has a sublinear regret such that $\frac{R^{(t)}}{t} = \Ocal\parenth{\sqrt{ \frac{\ln t}{t}}}$.
\end{corollary}
\begin{proof}
By definition of the negative entropy, we have $f_\phi(x) = (x-\omega) \ln (x-\omega) - (1-\omega)\ln(1-\omega) = \Ocal(x \ln x)$, and the results follows by Theorem~\ref{thm:f-div-sublinear-regret}.1.
\end{proof}

In fact, we can obtain a more explicit upper bound on the regret. By Proposition~\ref{prop:f-div-regret}, we have
\al{
\frac{R^{(t)}}{t} \leq \frac{(1-\omega)M^2}{2} \frac{\sum_{\tau = 1}^t \eta_{\tau+1}}{t} + {L} D(S) d_t + \frac{1}{t \eta_{t+1}} vd_t^nf_\phi\parenth{\frac{1}{vd_t^n}}
}
and we bound $f_\phi\parenth{\frac{1}{v d_t^n}}$. First, we have
\al{
f_\phi(x) 
&= (x-\omega) \ln(x-\omega) - (1-\omega)\ln(1-\omega) \\
&\leq (x-\omega) \ln(x-\omega) \\
&\leq c_\omega x \ln x + d_\omega \hspace{3cm} &\text{for $x \geq 1$}
}
where $c_\omega = 1+\ln(1-\omega)$ and $d_\omega = (1-\omega)\ln(1-\omega)$. To prove the last inequality, let $\delta(x) = (x-\omega) \ln(x-\omega) - c_\omega x \ln x + d_\omega$. Then $\delta(1) = 0$ and for all $x \geq 1$,
\al{
\delta'(x) &= \ln(x-\omega) + 1 - c_\omega(1+\ln x) = \ln \frac{x-\omega}{x^{c_\omega}(1-\omega)} 
\leq \ln \frac{x(1-\omega)}{x^{c_\omega}(1-\omega)} = \ln x^{1-c_\omega} \leq 0
}

Thus, whenever $d_t \leq 1$, we have $\frac{1}{vd_t^n} \geq 1$, and
\al{
vd_t^n f_\phi\parenth{\frac{1}{vd_t^n}}  
&\leq c_\omega \ln \frac{1}{vd_t^n} + d_\omega v d_t^n
}
and taking $d_t = 1/t$, we have
\[
\frac{R^{(t)}}{t} \leq \frac{(1-\omega)M^2}{2} \frac{\sum_{\tau = 1}^t \eta_{\tau+1}}{t} + \frac{{L}D(S)}{t} + \frac{c_\omega(n\ln t - \ln v) + d_\omega v / t^n}{t \, \eta_{t+1}}
\]
In particular, when $\omega = 0$ (i.e. for the Hedge algorithm), $c_\omega = 1$ and $d_\omega = 0$, and the bound simplifies to
\[
\frac{R^{(t)}}{t} \leq \frac{M^2}{2} \frac{\sum_{\tau = 1}^t \eta_{\tau+1}}{t} + \frac{{L}D(S)}{t} + \frac{n\ln t - \ln v}{t \, \eta_{t+1}}.
\]
and we recover the bound on the Hedge regret obtained in~\cite{krichene2015hedge}.
%Finally, taking $\eta_t = \theta (\ln t)^{\frac{1}{2}} t^{-\frac{1}{2}}$ and bounding $\sum_{\tau = 1}^t \eta_\tau \leq \theta \ln t 2 \sqrt{t}$,
%\[
%\frac{R^{(t)}}{t} \leq \theta M^2\sqrt{\frac{\ln t}{t}} + \frac{LD(S)}{t} + \frac{n \ln t - \ln v}{\sqrt{t \ln t}}
%\]

%============================================================================================
\section{Concluding Remarks}
\label{sec:conclusion}
We studied a sequential problem in which a decision maker chooses, at each iteration, a distribution $x^{(t)}$ over a compact set $S$, then observes a loss $\ell^{(t)}$ from the class of $L$-Lipschitz continuous functions on $S$. Viewing the problem as an online convex problem over $L^2(S)$, we applied the dual averaging method and derived a general regret bound. Then we studied dual averaging with Csisz\'ar divergences induced by $\omega$-potentials, and showed that for this class, the Bregman projection~\eqref{eq:dual_averaging_update} can be computed efficiently, assuming one can efficiently evaluate integrals on $S$ (e.g. by using a MCMC method). We then provided sufficient conditions on the asymptotic behavior of the potential to guarantee (i) strong convexity of the Csisz\'ar divergence (Theorem~\ref{thm:DA_strong_convexity}), and (ii) a sublinear regret (Theorem~\ref{thm:f-div-sublinear-regret}). These sufficient conditions provide guidance in the design and analysis of dual averaging methods, which we illustrated with one particular family of $\omega$-potentials, using entropy regularizers.

Another approach for learning on a continuum consists in applying a discrete learning algorithm on a finite cover of $S$. Since the loss functions are ${L}$-Lipschitz, the additional regret incurred due to learning on the cover is at most ${L}$ times the diameter of each element of the cover. Therefore, in order to have asymptotically sublinear regret, one would need to refine the cover as the number of iterations grows. Our method does not require explicitly computing a cover, since it samples directly from a distribution defined on $S$. It has the potential of being more computationally tractable (since one does not need to compute and refine a cover) but ultimately, its complexity is that of sampling from the sequence of distributions $(x^{(t)})$, thus studying the computational complexity of the proposed dual averaging method requires making additional assumptions on the family of loss functions and the feasible set $S$.

A related problem is bandit learning, in which the decision maker plays, at iteration $t$, an action $s^{(t)}$ drawn from the distribution $x^{(t)}$, then only observes the loss of that action, $\ell^{(t)}(s^{(t)})$, as opposed to the full loss function $\ell^{(t)}$. This problem is studied for example in~\cite{bubeck2011x}, for Lipschitz losses, and when the feasible set $S$ is given by an explicit hierarchical formulation. We are currently investigating extensions of this work to the bandit setting.

%============================================================================================
\newpage
\bibliography{DA-arXiv}
\bibliographystyle{plain}

%============================================================================================
\appendix
\section{Omitted proofs}

\subsection{Proof of Proposition~\ref{prop:smoothness_duality}}
\label{sec:appendix_proofs}

%-----------------------------------------------------------------------------------------------------------------------------------------------------------------
\paragraph*{Proposition~\ref{prop:smoothness_duality}}
If $\psi$ is $\ell_\psi$-strongly convex with respect to $\|\cdot \|$, then $\psi^*$ is $\frac{1}{\ell_\psi}$-smooth with respect to the dual norm $\| \cdot \|_*$.

\begin{proof}
Let $y_1, y_2 \in E^*$, and $x_i  = \nabla \psi^*(y_i)$. Since $x_i$ is the minimizer of the convex function $x \mapsto \psi(x) - \braket{y_i}{x}$, we have, by first-order optimality,
\al{
\braket{\nabla \psi(x_i) - y_i}{x - x_i} \geq 0 \ \forall x \in \Xcal
}
In particular, we have
\al{
\braket{\nabla \psi(x_1) - y_1}{x_2 - x_1} \geq 0 \\
\braket{\nabla \psi(x_2) - y_2}{x_1 - x_2} \geq 0
}
and summing both inequalities,
\[
\braket{y_2 - y_1}{x_2 - x_1} \geq \braket{\nabla \psi(x_2) - \nabla \psi(x_1)}{x_2 - x_1}
\]
By definition of the Bregman divergence, we have
\al{
\braket{y_2 - y_1}{x_2 - x_1}
&= \braket{\nabla \psi(x_2) - \nabla \psi(x_1)}{x_2 - x_1} \\
&= D_\psi(x_1, x_2) + D_\psi(x_2, x_1) \\
&\geq \ell_\psi \|x_2 - x_1\|^2
}
and by definition of the dual norm, we have $\braket{y_2 - y_1}{x_2 - x_1} \leq \|y_2 - y_1\|_* \|x_2 - x_1\|$, since $\|y\|^* = \sup_{\|x\| \leq 1} \braket{x}{y} \geq \braket{\frac{x}{\|x\|}}{y}$. Therefore,
\[
\|y_2 - y_1\|_* \|x_2 - x_1\| \geq \ell_\psi \|x_2 - x_1\|^2
\]
rearranging, we have $\|x_2 - x_1\| \leq \frac{1}{\ell_\psi} \|y_2 - y_1\|_*$, i.e.
\begin{equation}
\label{eq:lipschitz_gradient}
\|\nabla \psi^*(y_2) - \nabla \psi^*(y_1)\| \leq \frac{1}{\ell_\psi} \| y_2 - y_1\|_*
\end{equation}
Now by definition of the Bregman divergence, we have
\al{
D_{\psi^*}(x, y) 
&= \psi^*(x) - \psi^*(y) - \braket{\nabla \psi^*(y)}{x - y} \\
&= \int_0^1 \braket{\nabla \psi^*(y + t(x-y)) - \nabla \psi^*(y) }{x - y} dt \\
&\leq \|y - x\|_* \int_0^1 \|\nabla \psi^*(y+ t(x-y)) - \nabla \psi^*(y)\| dt \\
&\leq \|y - x\|_* \int_0^1 \frac{1}{\ell_\psi} \|y + t(x-y) - y\|_* dt &\text{by~\eqref{eq:lipschitz_gradient}}\\
&\leq \frac{1}{\ell_\psi} \|x - y\|_*^{2} \int_0^1t dt \\
&= \frac{1}{\ell_\psi} \|x - y\|_*^{2} \frac{1}{2}
}

\end{proof}

\subsection{Proof of Theorem~\ref{thm:f-div-sublinear-regret}.1}
\paragraph{Theorem \ref{thm:f-div-sublinear-regret}.1}
Under the assumptions of Theorem~\ref{thm:f-div-sublinear-regret}, suppose that there exists $\epsilon > 0$, $\nu \geq 0$ such that
\[
f_\phi(x) = \Ocal\parenth{ x^{1+\epsilon} (\ln x)^\nu }
\]
and the learning rates are taken to be $\eta_t = \Theta\parenth{(\ln t)^{\frac{\nu}{2}}t^{-\alpha}}$, $\alpha = \frac{1}{2+n\epsilon}$ then
\[
\frac{R^{(t)}}{t} = \Ocal\parenth{ (\ln t)^{\frac{\nu}{2}}t^{-\frac{1}{2+n\epsilon}} }.
\]
\begin{proof}
The proof is similar to that of Theorem~\ref{thm:f-div-sublinear-regret}. Let $(d_t)$ be a decreasing sequence which converges to $0$. By Proposition~\ref{prop:f-div-regret},
\[
\frac{R^{(t)}}{t} \leq \frac{M^2}{2\ell_\psi} \frac{\sum_{\tau = 1}^t \eta_\tau}{t} + {L} D(S) d_t + \frac{1}{t\eta_t} vd_t^nf_\phi\parenth{\frac{1}{vd_t^n}} 
\]
where $\sum_{\tau = 1}^t \eta_\tau = \Ocal((\ln t)^\frac{\nu}{2} t^{1-\alpha}$, $\frac{1}{t\eta_t} = \Ocal\parenth{ (\ln t)^{-\frac{\nu}{2}} t^{\alpha - 1}}$, and by assumption on $f_\phi$,
\al{
vd_t^n f_\phi\parenth{\frac{1}{vd_t^n}} \leq C v^{-\epsilon} d_t^{-n\epsilon} \parenth{\ln \frac{1}{v} + n\ln \frac{1}{d_t}}^\nu
= \Ocal\parenth{d_t^{-n\epsilon} \parenth{\ln \frac{1}{d_t}}^\nu}
}
Combining these bounds, we have
\al{
\frac{R^{(t)}}{t} 
&= \Ocal\parenth{ 
(\ln t)^\frac{\nu}{2} t^{-\alpha} + d_t + d_t^{-n\epsilon} t^{-(1 - \alpha)}(\ln t)^{-\frac{\nu}{2}} \parenth{\ln \frac{1}{d_t}}^\nu}
}
Taking $d_t = t^{-\frac{1-\alpha}{1+n\epsilon}}$, the bound becomes 
\al{
\frac{R^{(t)}}{t} 
&=\Ocal\parenth{ 
(\ln t)^\frac{\nu}{2} t^{-\alpha} + t^{-\frac{(1-\alpha)}{1+n\epsilon}}(\ln t)^{\frac{\nu}{2}}
}
}
which proves the claim.
\end{proof}

\end{document}